\documentclass[sigconf]{acmart}
\AtBeginDocument{%
  \providecommand\BibTeX{{%
    \normalfont B\kern-0.5em{\scshape i\kern-0.25em b}\kern-0.8em\TeX}}}

\usepackage{booktabs}
\setcopyright{acmcopyright}
\copyrightyear{2022}
\acmYear{2022}
\acmDOI{XXXXXXX.XXXXXXX}

\acmConference[KDD '22]{August 14-18, 2022, Washington DC, USA}{August 14-18,
  2022}{Washington DC}
\acmPrice{15.00}
\acmISBN{978-1-4503-XXXX-X/18/06}

\usepackage{todonotes}
\usepackage{amsmath}
\DeclareMathOperator*{\argmin}{argmin}
\newcommand{\ip}[2]{\left\langle #1, #2 \right \rangle}
\newtheorem{problem}{Problem}[section]

\begin{document}

\title{HyperAid: Denoising in hyperbolic spaces for tree-fitting and hierarchical clustering}


\author{Eli Chien, Puoya Tabaghi, Olgica Milenkovic}
\affiliation{%
  \institution{University of Illinois Urbana-Champaign}
  \country{}
}
\email{{ichien3,tabaghi2,milenkov}@illinois.edu}







\renewcommand{\shortauthors}{Chien, et al.}

\begin{abstract}
  The problem of fitting distances by tree-metrics has received significant attention in the theoretical computer science and machine learning communities alike, due to many applications in natural language processing, phylogeny, cancer genomics and a myriad of problem areas that involve hierarchical clustering. Despite the existence of several provably exact algorithms for tree-metric fitting of data that inherently obeys tree-metric constraints, much less is known about how to best fit tree-metrics for data whose structure moderately (or substantially) differs from a tree. For such noisy data, most available algorithms perform poorly and often produce negative edge weights in representative trees. Furthermore, it is currently not known how to choose the most suitable approximation objective for noisy fitting. Our contributions are as follows. First, we propose a new approach to tree-metric denoising (HyperAid) in hyperbolic spaces which transforms the original data into data that is ``more'' tree-like, when evaluated in terms of Gromov's $\delta$ hyperbolicity. Second, we perform an ablation study involving two choices for the approximation objective, $\ell_p$ norms and the Dasgupta loss. Third, we integrate HyperAid with schemes for enforcing nonnegative edge-weights. As a result, the HyperAid platform outperforms all other existing methods in the literature, including Neighbor Joining (NJ), TreeRep and T-REX, both on synthetic and real-world data. Synthetic data is represented by edge-augmented trees and shortest-distance metrics while the real-world datasets include Zoo, Iris, Glass, Segmentation and SpamBase; on these datasets, the average improvement with respect to NJ is $125.94\%$. Our code is publicly available.\footnote{\url{https://github.com/elichienxD/HyperAid}} 
\end{abstract}




\begin{teaserfigure}
  \includegraphics[clip,width=\textwidth]{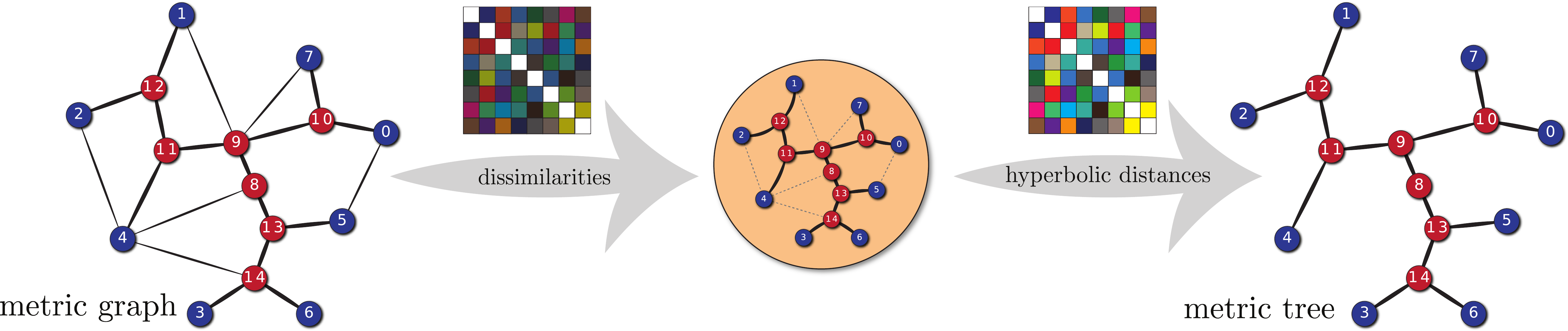}
  \caption{The HyperAid framework.}
  \label{fig:teaser}
\end{teaserfigure}

\maketitle

\section{Introduction}
The problem of fitting tree-distances between data points is of great relevance in application areas concerned with extracting and processing hierarchical information~\cite{saitou1987neighbor,ailon2005fitting,macaluso2021cluster,coenen2019visualizing,hewitt2019structural}. A significant body of work on the topic has focused on fitting phylogenetic trees based on evolutionary distances~\cite{saitou1987neighbor} and using fitted tree-models for historic linguistics and natural language processing~\cite{coenen2019visualizing,hewitt2019structural}. Since trees are usually hard to accurately represent in low-dimensional Euclidean spaces~\cite{linial1995geometry}, alternative methods have been suggested for representation learning in hyperbolic spaces~\cite{nickel2017poincare}. For these and related approaches, it has been observed that the quality of embedding and hence the representation error depend on the so called Gromov $\delta$-hyperbolicity~\cite{gromov1987hyperbolic} of the data, which will be discussed in-depth in the following section. For perfect trees, $\delta=0$, and the larger the value of $\delta$ the less tree-like the data structure. When $\delta$ is sufficiently large, as is the case when dealing with noisy observations or observations with outliers, the embedding distortion may be significant~\cite{abraham2007reconstructing}. This is a well-known fact that has limited the performance of methods such as Neighbor Joining (NJ) and TreeRep, which come with provable reconstruction guarantees only in the presence of small tree-metric perturbations/noise~\cite{saitou1987neighbor,sonthalia2020tree} (more precisely, the correct output tree topology is guaranteed for ``nearly additive'' distance matrices, i.e., matrices in which every distance differs from the true distance by not more than half of the shortest edge weight in the tree). Heuristic methods for tree-metric fitting, such as NINJA~\cite{wheeler2009large} and T-REX~\cite{boc2012t} have similar issues, but address the problems of scaling, and in the latter case, the issue of negative tree-edge weights that may arise due to large $\delta$-hyperbolicity.

In a parallel development, the theoretical computer science community has explored the problem of fitting both general tree-metrics and ultrametrics with specialized distortion objectives~\cite{ailon2005fitting,harb2005approximating,cohen2021fitting}. Ultrametrics are induced by rooted trees for which all the root-to-leaf distances are the same. Although ultrametrics do not have the versatility to model nonlinear evolutionary phenomena, they are of great relevance in data analysis, in particular in the context of ``linkage'' clustering algorithms (single, complete or average linkage)~\cite{ailon2005fitting}. Existing methods for ultrametric approximations scale well and provide approximation guarantees for their general graph problem counterparts. For both general tree-metrics and ultrametrics, two objectives have received significant attention: The Dasgupta objective and derivatives thereof~\cite{dasgupta2016cost,charikar2019hierarchical} and the $\ell_p$ objective~\cite{ailon2005fitting,harb2005approximating}. Dasgupta's objective was notably used for hierarchical clustering (HC) in hyperbolic spaces as a means to mitigate the hard combinatorial optimization questions associated with this clustering metric~\cite{sahoo2020tree}; see also~\cite{macaluso2021cluster} for other approaches to HC in hyperbolic spaces. Neither of the aforementioned approaches provides a solution to the problem of fitting tree-metrics for data with moderate-to-high Gromov hyperbolicity, i.e., noisy or perturbed tree-metric data not covered by existing performance guarantee analyses. In addition, no previous study has compared the effectiveness of different objectives for the measure of fit for tree-metrics on real-world datasets.

Our contributions are three-fold.
\begin{itemize}
    \item We propose the first approach to data denoising (i.e., Gromov hyperbolicty reduction) in hyperbolic space as the first step to improving the quality of tree-metric fits.
    \item We combine our preprocessing method designed to fit both general trees and ultrametrics with solutions for removing negative edge-weights in the tree-metrics that arise due to practical deviations from tree-like structures.
    \item We demonstrate the influence of hyperbolic denoising on the quality of tree-metric fitting and hierachical clustering for different objectives. In particular, we perform an ablation study that illustrates the advantages and disadvantages of different objectives (i.e., Dasgupta~\cite{dasgupta2016cost} and $\ell_2$~\cite{harb2005approximating}) when combined with denoising in hyperbolic spaces.
\end{itemize}

The paper is organized as follows. Related works are discussed in Section~\ref{sec:related_works} while the relevant terminology and mathematical background material are presented in Section~\ref{sec:prelim}. Our tree-fitting, HC clustering approach and accompanying analysis on different choices of objective functions are described in Section~\ref{sec:obj_choice}. More details regarding the HyperAid method are provided in Section~\ref{sec:HyperAid}, with experimental results made available in Section~\ref{sec:exp}. All proofs are relegated to the Supplement.

\section{Related works}\label{sec:related_works}

\textbf{The Dasgupta HC objective. } Determining the quality of trees produced by HC algorithms is greatly aided by an adequate choice of an objective function. Dasgupta~\cite{dasgupta2016cost} proposed an objective and proved that an $O(\alpha_n\log n)$-approximation can be obtained using specialized $\alpha_n$-approximation sparsest cut procedures~\cite{dasgupta2016cost}. HC through the lens of Dasgupta's objective and variants thereof have been investigated in~\cite{charikar2017approximate,moseley2017approximation,ahmadian2019bisect,alon2020hierarchical}. For example, the authors of~\cite{moseley2017approximation} showed that the average linkage (UPGMA) algorithm produces a $1/3$-approximation for a Dasgupta-type objective. It is important to point out that the Dasgupta objective depends only on the sizes of unweighted subtrees rooted at lowest common ancestors (LCAs) of two leaves, the optimal trees are restricted to be binary, and that it does not produce edge-lengths (weights) in the hierarchy. The latter two properties are undesirable for applications such as natural language processing, phylogeny and evolutionary biology in general~\cite{waterman1977additive}. As a result, other objectives have been considered, most notably metric-tree fitting objectives, which we show can resolve certain issues associated with the use of the Dasgupta objective.

\textbf{Tree-metric learning and HC. }In comparison to the Dasgupta objective, metric-based objective functions for HC have been mostly overlooked. One metric based objective is the $\ell_p$ norm of the difference between the input metric and the resulting (fitted) tree-metric~\cite{waterman1977additive}. For example, the authors of~\cite{farach1995robust} showed that an optimal ultrametric can be found with respect to the $\ell_\infty$ norm loss for arbitrary input metric in polynomial time. It is also known in the literature that finding optimal tree-metrics with respect to the $\ell_p$ norm, for $p\in[1,\infty],$ is NP-hard and finding optimal ultrametrics for $p\in[1,\infty)$ is also NP-hard. In fact, the $\ell_\infty$ problem for tree-metrics and the $\ell_1$ problem for tree-metrics and ultrametrics are both APX-hard~\cite{wareham1993complexity,agarwala1998approximability,ailon2005fitting}. The work~\cite{agarwala1998approximability} presents a $3$-approximation result for the closest tree-metric under $\ell_\infty$ norm and established an important connection between the problems of fitting an ultrametric and fitting a tree-metric. The result shows that an $\alpha$-approximation result for the constrained ultrametric fitting problem yields a $3\alpha$-approximation solution for tree-metric fitting problem. The majority of the known results for tree-metric fitting depend on this connection. The algorithm in~\cite{ailon2005fitting} offers an $O((\log n\log\log n)^{1/p})$ approximation for both tree-metrics fitting and ultrametrics fitting under general $\ell_p$ norms. The work~\cite{cohen2021fitting} improves the approximation factor to $O(1)$ for the $\ell_1$ norm. Unfortunately, despite having provable performance guarantees, both the methods from~\cite{ailon2005fitting} and~\cite{cohen2021fitting} are not scalable due to the use of correlation clustering. As a result, they cannot be executed efficiently in practice even for trees that have $\sim$50 leaves, due to severe computational bottlenecks. 

\textbf{Gradient-based HC.} The authors of~\cite{chierchia2020ultrametric} proposed an ultrametric fitting framework (Ufit) for HC based on gradient descent methods that minimize the $\ell_2$ loss between the resulting ultrametric and input metric. There, a subdifferentiable min-max operation is integrated into the cost optimization framework to guarantee ultrametricity. Since an ultrametric is a special case of a tree-metric, Ufit can be used in place of general tree-metric fitting methods. The gHHC method~\cite{monath2019gradient} represents the first HC approach based on hyperbolic embeddings. The authors assume that the hyperbolic embeddings of leaves are given and describe how to learn the internal vertex embeddings by optimizing a specialized objective function. The HypHC approach~\cite{chami2020trees} disposes of the requirement to have the leaf embeddings. It directly learns a hyperbolic embedding by optimizing a relaxation of Dasgupta's objective~\cite{dasgupta2016cost} and reconstructs the discrete binary tree from the learned hyperbolic embeddings of leaves using the distance of their hyperbolic lowest common ancestors to the root. Compared to gHHC, our framework also does not require input hyperbolic embeddings of leaves. We learn the hyperbolic embeddings of leaves by directly minimizing the $\ell_2$ loss (rather than the Dasgupta objective) between the hyperbolic distances and the input metric of pairs of points, which substantially differs from HypHC and has the goal to preprocess (denoise) the data for downstream tree-metric fitting. Furthermore, it will be shown in the subsequent exposition that metric based losses are preferable to the Dasgupta objective for the new task of hyperbolic denoising and consequent HC. 

\textbf{Learning in hyperbolic spaces. }Representation learning in hyperbolic spaces has received
significant interest due to its effectiveness in capturing latent hierarchical structures~\cite{krioukov2010hyperbolic,nickel2017poincare,papadopoulos2015network,tifrea2018poincare}. Several learning methods for Euclidean spaces have been generalized to hyperbolic spaces, including perceptrons and SVMs~\cite{cho2019large,chien2021highly,tabaghi2021linear}, neural networks~\cite{ganea2018hyperbolic,shimizu2021hyperbolic} and graph neural networks~\cite{chami2019hyperbolic,liu2019hyperbolic}. In the context of learning hyperbolic embeddings, the author of~\cite{sarkar2011low} proposed a combinatorial approach for embedding trees with low distortion in just two dimensions. The work~\cite{pmlr-v80-sala18a} extended this idea to higher dimensions. Learning methods for hyperbolic embeddings via gradient-based techniques were discussed in~\cite{nickel2017poincare,nickel2018learning}. None of these methods currently provides quality clustering results for practical data whose distances do not ``closely'' match those induced by a tree-metric. Furthermore, the main focus of~\cite{nickel2017poincare,nickel2018learning} is to learn hyperbolic embeddings that represent the input graphs and trees as accurately as possible. In contrast, our goal is to learn hyperbolic metrics that effectively ``denoise'' parts of the input metric to better conform tree structures, in order to improve the performance of arbitrary metric-tree fitting and HC downstream methods.

\section{Preliminaries}\label{sec:prelim}
\begin{figure*}
  \includegraphics[trim={2.5cm 9cm 10.5cm 3cm},clip,width=0.3\linewidth]{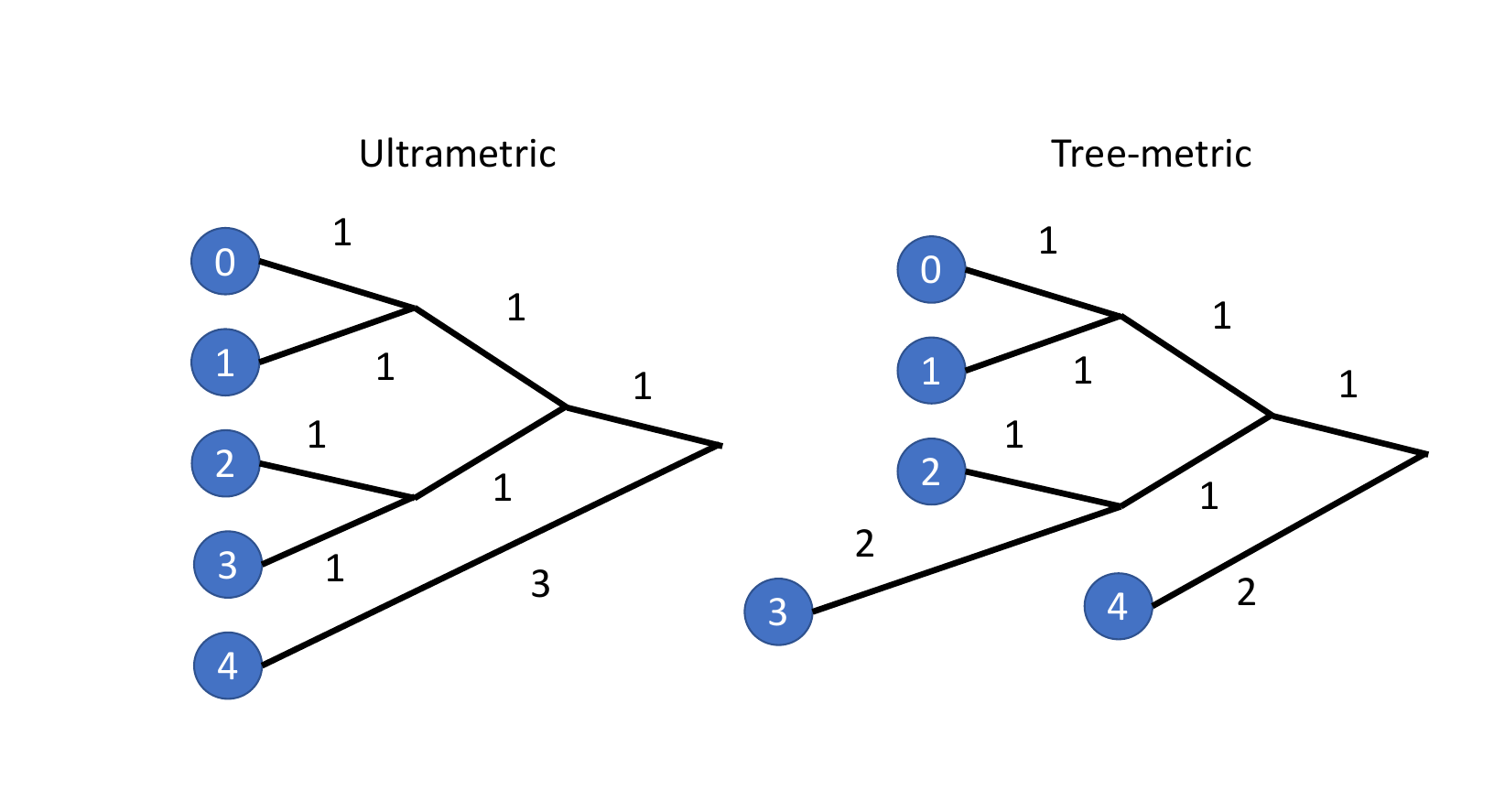}
  \includegraphics[trim={1.5cm 9cm 2.5cm 3cm},clip,width=0.68\linewidth]{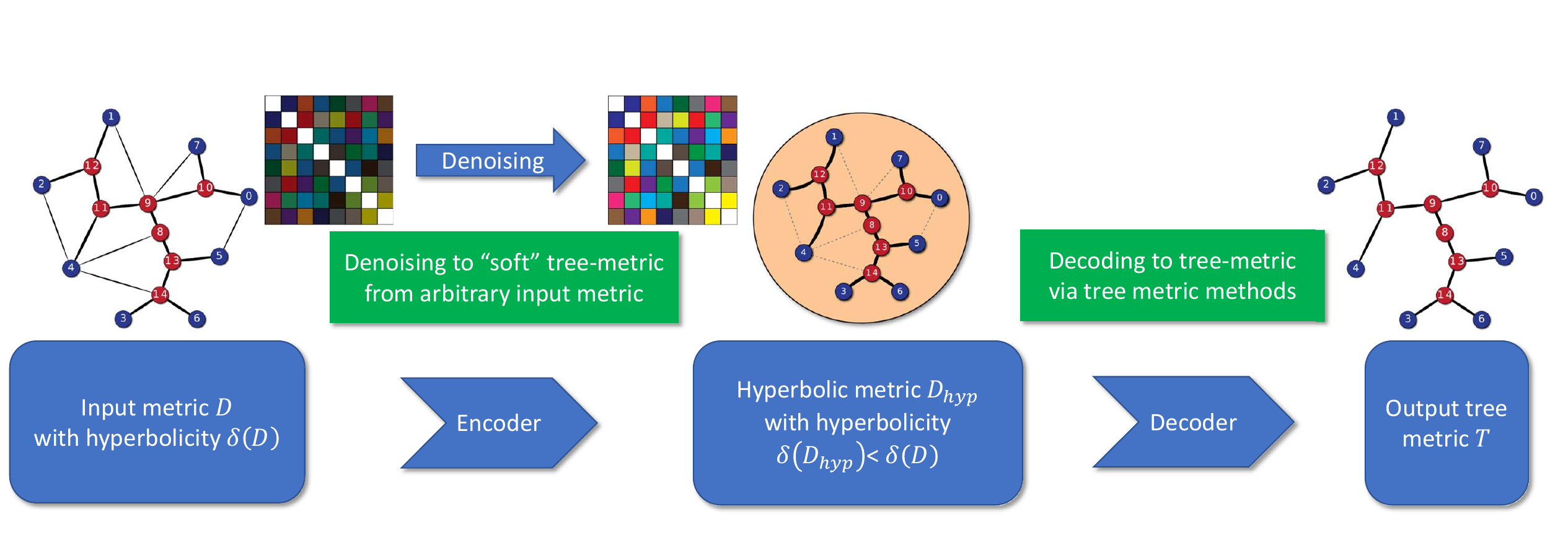}
  \caption{Left: Illustration of the generative tree for an ultrametric and a tree-metric. Note that all the weighted distances of leaves to the root are equal for an ultrametric, while general tree-metrics do not have to satisfy this constraint. Right: Detailed block diagram of our HyperAid framework. Detailed explanations of the encoder and decoder modules are available in Section~\ref{sec:HyperAid}.}
  \label{fig:tree_ultrametric}
\end{figure*}

\textbf{Notation and terminology. }Let $T = (V,E,w)$ be a tree with vertex set $V$, edge set $E$ and nonnegative edge weights $w:E \rightarrow \mathbb{R}_{+}$. Let $\mathcal{L}(T)\subseteq V$ denote the set of leaf vertices in $T$. The shortest length path between any two vertices $v_1, v_2 \in V$ is given by the metric $d_{T}: V \times V \rightarrow \mathbb{R}_{+}$. The metric $d_T$ is called a tree-metric induced by the (weighted) tree $T$. In binary trees, there is only one internal (nonleaf) vertex of degree two, the root vertex $v_r$; all other internal vertices are of degree three. Since our edge weights are allowed to be zero, nonbinary trees can be obtained by adding edges of zero weight. The pairwise distances of leaves in this case remain unchanged. A vertex $v_2$ is said to be a descendant of vertex $v_1$ if $v_1$ belongs to the directed path from the root $v_r$ to $v_2$. Also, with this definition, the vertex $v_1$ is an ancestor of the vertex $v_2$. For any two vertices $v_1, v_2 \in \mathcal{L}(T)$, we let $v_1 \land v_2 $ denote their lowest common ancestor (LCA), i.e., the common ancestor that is the farthest from the root vertex. We define a \textit{clan} in $T$ as a subset of \textit{leaf} vertices that can be separated from the rest of the tree by removing a single edge. Furthermore, we let $\mathcal{T}_{n}$ be the set of binary trees with $n$ leaves. For any internal vertex $v$ of $T \in \mathcal{T}_n$, we let $T(v)$ be a binary subtree of $T$ rooted at $v$. 

\textbf{$\delta$-hyperbolic metrics. }Gromov introduced the notion of $\delta$-hyperbolic metrics as a generalization of the type of metric obtained from manifolds with constant negative curvature~\cite{gromov1987hyperbolic,sonthalia2020tree}, as described below.

\begin{definition}
Given a metric space $(V,d)$, the Gromov product of $x,y\in V$ with respect to a base point $r\in V$ is
\begin{align}
    (x,y)_r = \frac{1}{2}\left(d(r,x)+d(r,y)-d(x,y)\right).
\end{align}
\end{definition}
Note that the Gromov product measures how close $r$ is to the geodesic connecting $x$ and $y$~\cite{sonthalia2020tree}.
\begin{definition}\label{def:gromov_hyperbolicity}
A metric $d$ on a space $V$ is a $\delta$-hyperbolic metric for $\delta\geq 0$ if
\begin{align}\label{eq:four_points}
    \forall r,x,y,z\in V\quad (x,y)_r \geq \min\left((x,z)_r,(y,z)_r \right) - \delta.
\end{align}
\end{definition}
Usually, when stating that $d$ is $\delta$-hyperbolic we mean that $\delta$ is the smallest possible value that satisfies the condition~\eqref{eq:four_points}. An ultrametric is a special case of a tree-metric~\cite{ailon2005fitting}, and is formally defined as follows.
\begin{definition}
 A metric $d_U$ on a space $V$ is called an ultrametric if it satisfies the strong triangle inequality property
 \begin{align}
     \forall x,y,z \in V,\quad d_U(x,y) \leq \max \left\{ d_U(x,z),d_U(y,z)\right\}.
 \end{align}
\end{definition}
Note that the generating tree of an ultrametric $d_U$ on $V$ has a special structure: All element in $V$ are leaves of the tree and all leaves are at the same distance from the root~\cite{ailon2005fitting}. An illustration of an ultrametric can be found at Figure~\ref{fig:tree_ultrametric}.

\textbf{Hyperbolic geometry. }A hyperbolic space is a nonEuclidean space with a constant negative curvature. Despite the existence of various equivalent models for hyperbolic spaces, Poincar\'e ball models have received the broadest attention in the machine learning and data mining communities. Although our hyperbolic denoising framework can be extended to work for any hyperbolic model, we choose to work with Poincar\'e ball models with negative curvature $-c,c>0$: $\mathbb{B}^d_c=\{x\in \mathbb{R}^d:\sqrt{c}\,\|x\|< 1\}$. For any two points $x,y\in \mathbb{B}^d_c$, their hyperbolic distance equals
\begin{align}
    d_{\mathbb{B}^d_c}(x,y) = \frac{1}{\sqrt{c}}\cosh^{-1}\left(1+\frac{2c\|x-y\|^2}{(1-c\|x\|^2)(1-c\|y\|^2)}\right).
\end{align}
Furthermore, for a reference point $r\in \mathbb{B}^d_c$, we denote its tangent space, the first order linear approximation of $\mathbb{B}^d_c$ around $r$, by $T_b\mathbb{B}^d_c$. M\"obius addition and scalar multiplication --- two basic operators on the Poincar\'e ball~\cite{ungar2008analytic} -- may be defined as follows. The M\"obius sum of $x,y\in\mathbb{B}^d_c$ equals
\begin{align}\label{eq:Mobius_add}
    & x \oplus_c y = \frac{(1+2c\ip{x}{y}+c\|y\|^2)x + (1-c\|x\|^2)y}{1+2c\ip{x}{y}+c^2\|x\|^2\|y\|^2}.
\end{align}
Unlike its vector-space counterpart, this addition is noncommutative and nonassociative. The M\"obius version of multiplication of $x\in \mathbb{B}^d_c\setminus\{0\}$ by a scalar $t\in \mathbb{R}$ is defined according to
\begin{align}\label{eq:Mobius_mul}
    & r \otimes_c x = \tanh(t\tanh^{-1}(\sqrt{c}\|x\|))\frac{x}{\sqrt{c}\|x\|} \text{ and }r\otimes_c 0 = 0.
\end{align}
For more details, see~\cite{vermeer2005geometric,ganea2018hyperbolic,chien2021highly,tabaghi2021linear}. 
With the same operators, one can also define geodesics -- analogues of straight lines in Euclidean spaces and shortest path in graphs -- in $\mathbb{B}^d_c$. The geodesics connecting two points $x,y\in \mathbb{B}^d_c$ is given by
\begin{align}\label{eq:geodesic}
    \gamma_{x\rightarrow y}(t) = x \oplus_c(t\otimes_c((-x)\oplus_c y)).
\end{align}
Note that $t \in [0,1]$ and $\gamma_{x\rightarrow y}(0) = x$ and $\gamma_{x\rightarrow y}(1) = y$. An illustration of the Poincare model is given in Figure~\ref{fig:poincare_disk}.

\begin{figure}
  \includegraphics[width=0.6\linewidth]{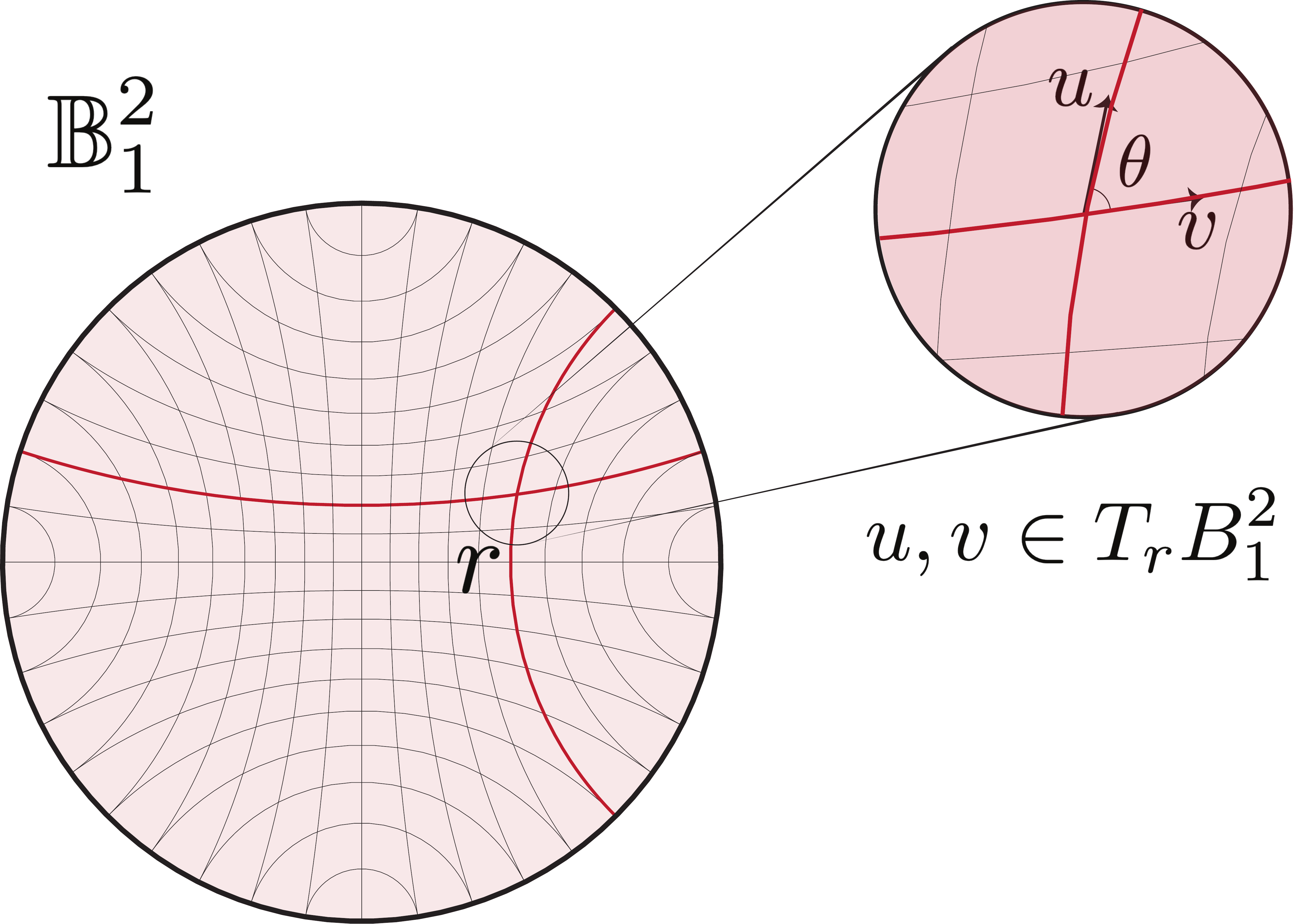}
  \caption{The two-dimensional Poincar\'e disk. (Right) a set of geodesics. (Left) vectors in the tangent space at $r \in \mathbb{B}_{1}^2$.}
  \label{fig:poincare_disk}
  \vspace{-0.25cm}
\end{figure}

\section{Hierarchical Clustering: A Tree Learning Problem}\label{sec:obj_choice}
As previously explained, tree-metric fitting is closely related to hierarchical clustering, and both methods are expected to greatly benefit from data denoising. It is nevertheless unclear for which tree-metric fitting algorithms and objectives used in hierarchical clustering does one see the most significant benefits (if any) from denoising. To this end, we describe two objectives that are commonly used for the aforementioned tasks in order to determine their strengths and weaknesses and potential performance benefits.

Let $\{ d_{i,j} \}_{i,j \in [n]}$ be a set of nonnegative dissimilarity scores between a set of $n$ entities. In HC problems, the goal is to find a tree $T$ that best represents a set of gradually finer hierarchical partitions of the entities. In this representation, the entities are placed at the leaves of the tree $T$, i.e., $|\mathcal{L}(T)| = n$. Each internal vertex $v$ of $T$ is representative of its clan, i.e., the leaves of $T(v)$ or the hierarchical cluster at the level of $v$. We find the following particular HC problem definition useful in our analysis.
\begin{problem}\label{prob:hierarchical_clustering}
Let $\{d_{i,j}\}_{i,j \in [n]}$ be a set of pairwise dissimilarities between $n$ entities. The optimal hierarchical clusters are equivalent to a binary tree $T^*$ that is a solution to the following problem:
    \begin{equation*}
        T^* = \argmin_{T \in \mathcal{T}_n} \sum_{i,j \in [n]} L(f(v_i, v_j) , d_{i,j} ),
    \end{equation*}
	where $f: \mathcal{L}(T) \times \mathcal{L}(T) \rightarrow \mathbb{R}_{+}$ is a (known) function of pairwise dissimilarities of the leaves in $T$, and $L(\cdot, \cdot)$ is a given loss function.
\end{problem}

\subsection{The Dasgupta objective}
Dasgupta~\cite{dasgupta2016cost} introduced the following cost function for HC:
	\begin{equation}\label{eq:dagupta_cost}
		\mathrm{cost}(T) = \sum_{i,j \in [n]} | \mathcal{L}( T(v_i \land v_j))|d_{i,j}.
	\end{equation}
The optimal HC solution is a binary tree that maximizes~\eqref{eq:dagupta_cost}. In the formalism of Problem \ref{prob:hierarchical_clustering}, this approach assumes that the pairwise dissimilarity between leaf vertices $v_i$ and $v_j$ equals the number of leaves of the tree rooted at their LCA, i.e.,
	\[
		f(v_i, v_j) = | \mathcal{L}( T(v_i \land v_j))|.
	\]
	In words, two dissimilar leaves $v_i$ and $v_j$ should have an LCA $v_i \land v_j$ that is close to the root. This ensures that the size of $| \mathcal{L}( T(v_i \land v_j))|$ is proportional to $d_{i,j}$. It is also important to note that this method is sensitive to the absolute value of the input dissimilarities, since the loss function is the inner product between measurements and tree leaf dissimilarities, i.e., $L(| \mathcal{L}( T(v_i \land v_j))| , d_{i,j} ) = -| \mathcal{L}( T(v_i \land v_j))| d_{i,j} $.

\subsection{The $\ell_p$ loss objective}

\begin{figure}[t]
  \includegraphics[width=\linewidth]{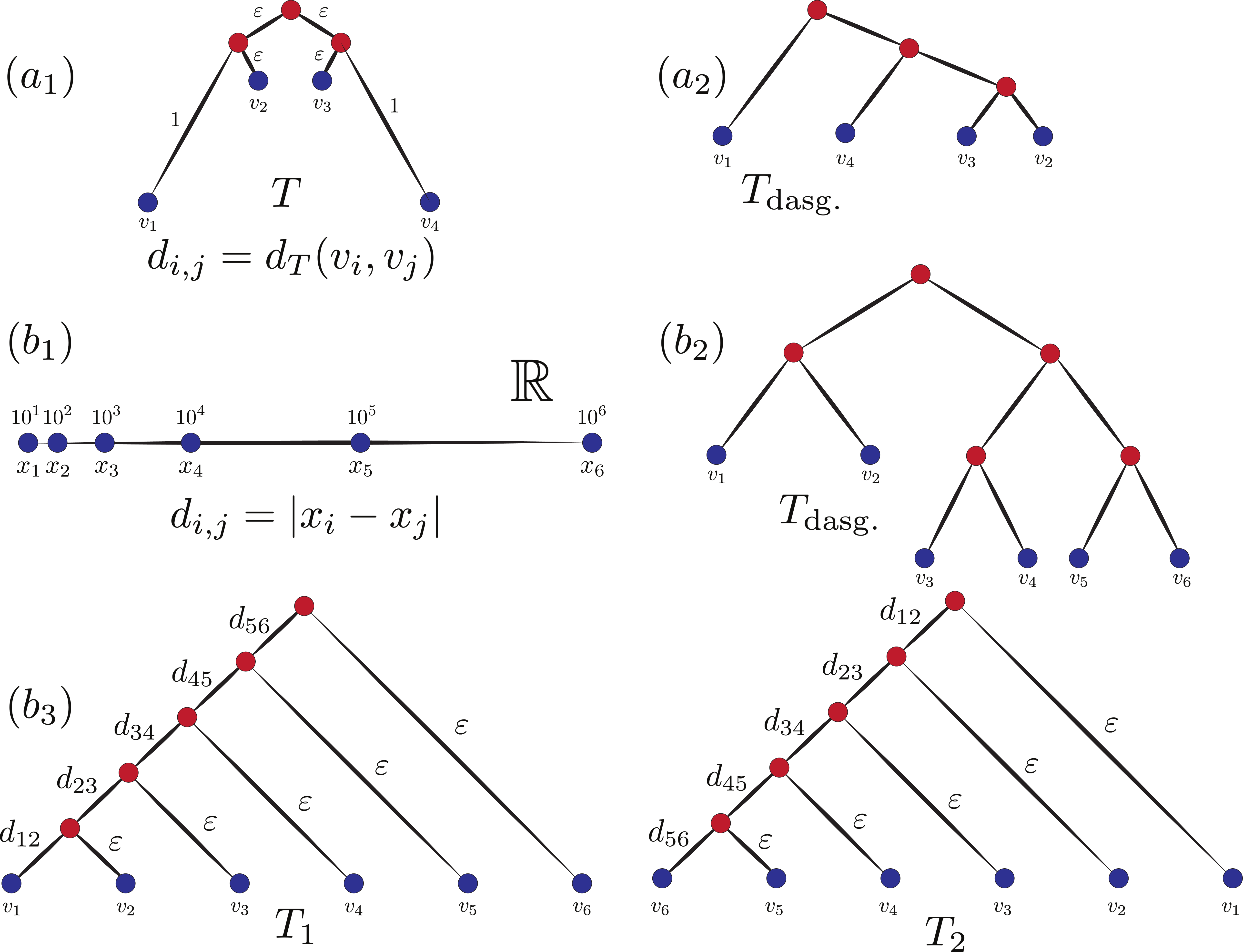}
  \caption{Hierarchical clusters learned by different objectives. Figure-$(a_1)$ Dissimilarity measurements are pairwise leaf distances on the weighted tree $T$ (optimal tree in $\ell_2$ sense). Figure-$(a_2)$: The tree $T_{\mathrm{dasg.}}$ optimizes Dasgupta's objective for measurements $\{ d_{T}(v_i,v_j)\}_{i,j \in [4]}$. Figure-$(b_1)$: Dissimilarity measurements are pairwise distances of points on a line in $\mathbb{R}$, i.e., $\mathcal{D} = \{ d_{i,j}: d_{i,j} = |x_i - x_j|, i,j \in [6] \},$ where $x_{i} = 10^{i}$ for all $i \in [6]$. Figure-$(b_2)$:   The Dasgupta-optimal tree for $\mathcal{D}$. Figure-$(b_3)$: Both trees $T_1$ and $T_2$ optimize the $\ell_p$ cost function ($\varepsilon \rightarrow 0^{+}$).}
  \label{fig:das_vs_l2}
  \vspace{-0.5cm}
\end{figure}

Instead of using the Dasgupta objective, one can choose to work with a leaf dissimilarity metric that is provided by the tree distance function, $d_{T}$. This way, one can learn a \emph{weighted tree} that best represents the input dissimilarity measurements. This is formally described as follows.
	\begin{problem}\label{prob:l_p}
	Let $\{d_{i,j}\}_{i,j \in [n]}$ be a set of pairwise dissimilarities of $n$ entities. The optimal hierarchical clusters in the $\ell_p$ norm sense are equivalent to a tree $T^*$ that minimizes the objective
		\begin{equation}\label{eq:hierarchical_cost}
		\mathrm{cost}_{p}(T) = \Big( \sum_{v_i,v_j \in \mathcal{L}(T) } |d_{T}(v_i, v_j)-d_{i,j}|^p \Big)^{\frac{1}{p}}, \text{ where $p \geq 1$.}
		\end{equation}
		
	\end{problem}
	According to our general formalism, Problem~\ref{prob:l_p} uses the natural tree-distance function to compute the leaf dissimilarities, i.e., $f(v_i, v_j) = d_{T}(v_i,v_j)$, and the $\ell_p$ loss to measure the quality of the learned trees, i.e., $L(d_{T}(v_i,v_j),d_{i,j}) = |d_{T}(v_i,v_j)-d_{i,j}|^{p}$. The objective function in Problem~\ref{prob:l_p} aims to jointly find a topology (adjacency matrix) and edge weights for a tree that best approximates a set of dissimilarity measurements. This problem can be interpreted as approximating the measurement with an additive distance matrix (i.e., distances induced by trees). This is a significantly stronger constraint than the one induced by the requirement that two entities that are similar to one another should be placed (as leaves) as close as possible in the tree. We illustrate this distinction in Figure~\ref{fig:das_vs_l2} $(a_1), (a_2)$: The dissimilarity measurements are generated according to the tree $T$ and the vertices $v_2$ and $v_3$ can be arbitrarily close to each other ($\varepsilon \rightarrow 0$)--- see Figure~\ref{fig:das_vs_l2}-$(a_1)$. The optimal tree in the $\ell_p$ sense does not place these vertices close to each other since the objective is to approximate all dissimilarity measurements with a set of additive distances. This is therefore in contrast to Dasgupta's objective where, for a small enough $\varepsilon$, it favors placing the vertices $v_2$ and $v_3$ closest to each other (see Figure~\ref{fig:das_vs_l2}-$(a_2)$). In Figures \ref{fig:das_vs_l2} $(b_1 - b_3)$, we provide another example to illustrate the distinction between optimal clusters according to the $\ell_2$ and Dasgupta's criteria.

\textbf{Remark 1. }We can interpret Problem~\ref{prob:l_p} as two independent problems: (1) Learning edge-weights, and (2) Learning the topology. A binary tree $T$ with $n$ leaves has a total of $2n-1$ vertices and $2n-2$ edges. Let $w = (w_1, \ldots, w_{2n-2})^{T} \in \mathbb{R}^{2n-2}_{+}$ be the vector of the edge weights. Any distance between pairs of vertices in $v_1, \ldots, v_n \in \mathcal{L}(T)$ can be expressed as a linear combination of the edge weights. We can hence write $d_T \in \mathbb{R}^{\frac{n \times (n-1)}{2}}$ for the vector of all pairwise distances of leaves, so that $d_T = A_T w$. Given a fixed topology (adjacency matrix of the tree $T$), we can compute the edges weight of $T$ that minimize the objective function in Problem~\ref{prob:l_p}
\begin{equation}
	\min_{w: E \rightarrow \mathbb{R}_{+} } \mathrm{cost}_{p}(T).
\end{equation}
There always exists at least one solution to the optimization problem \ref{prob:l_p}. This is the result of the following proposition.
\begin{proposition}\label{prop:convex}
The objective in Problem~\ref{prob:l_p} is a convex function of the edge weights of the tree $T$.	
\end{proposition}

\textbf{Remark 2. }Let $e_1$ and $e_2$ be two edges, with weights $w_1$ and $w_2$, incident to the root of a binary tree $T \in \mathcal{T}_{n}$. We can delete the root along with the edges $e_1$ and $e_2$, and place an edge of weight $w_1+w_2$ between its children. This process trims the binary tree $T$ while preserving all pairwise distances between its leaf vertices. Hence, we can also work with trimmed nonbinary trees. Furthermore, in this formalism, the root vertex does not carry any hierarchy-related information about the entities (leaves), although it is important for ``anchoring'' the tree in HC. We discuss how to select a root vertex upon tree reconstruction in our subsequent exposition.

\subsection{Hierarchical objectives: $\ell_p$ vs. Dasgupta}

\begin{figure}[t]
  \includegraphics[width=\linewidth]{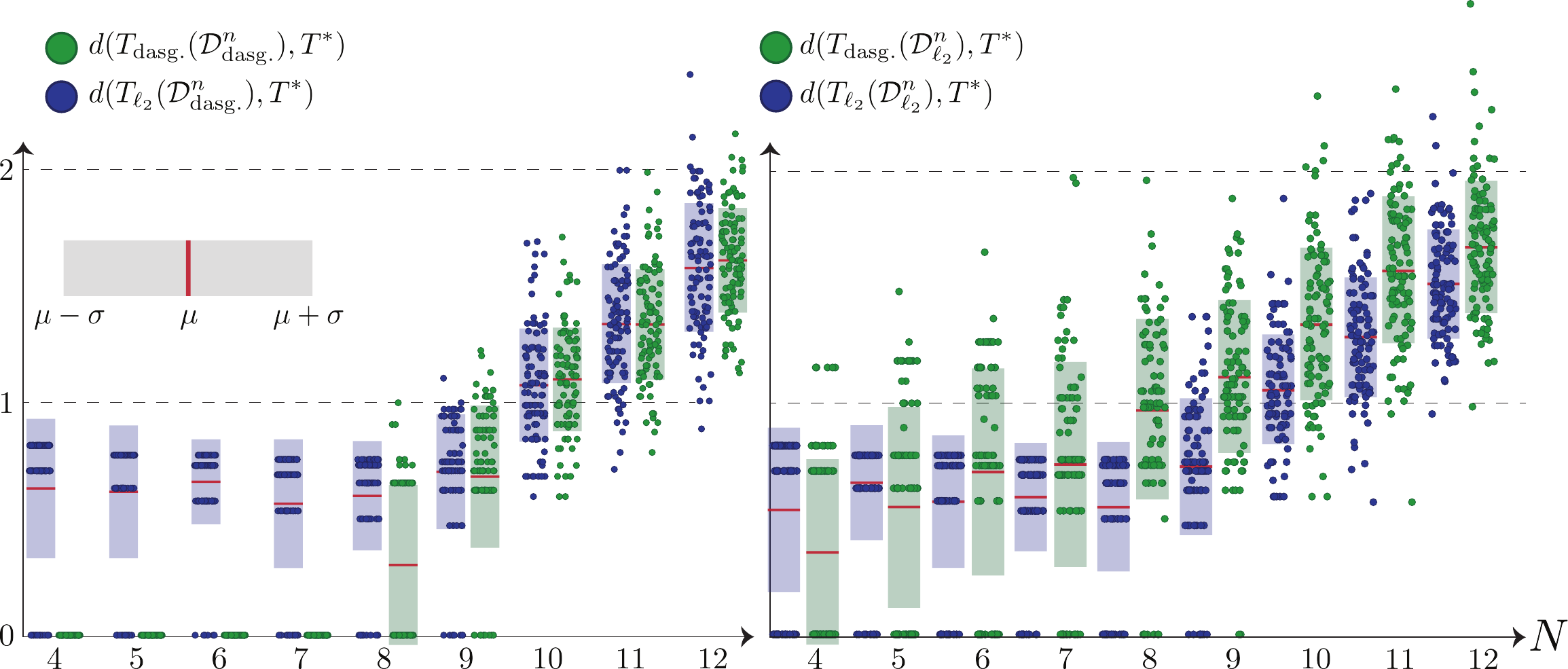}
  \caption{The reconstruction error for the trees $T_{\mathrm{dasg.}}(\mathcal{D}^{n}_{\mathrm{dasg.}})$, $T_{\mathrm{dasg.}}(\mathcal{D}^{n}_{\ell_2})$, $T_{\ell_2}(\mathcal{D}^{n}_{\ell_2})$, and $T_{\ell_2}(\mathcal{D}^{n}_{\mathrm{dasg.}})$. We have $d(T_1, T_{2}) := \frac{2}{n(n-1)}\|d_{T_1}-d_{T_{2}}\|_{2},$ where $d_T$ is the pairwise leaf distances in the tree $T \in \mathcal{T}_n$, assuming unit length edges.}
  \label{fig:exp_das_vs_l2}
\end{figure}

Here we discuss \emph{which objective function produces more ``acceptable'' trees and hierarchical clusters, $\ell_p$ or Dasgupta?} In clustering problems for which we do not have the ground truth clusters, it is hard to quantify the quality of the solution. This is especially the case when using different objectives, as in this case comparisons may be meaningless. To address this issue, we performed an experiment in which ground-truth synthetic data is modeled in a way that reflect the properties of the objectives used in optimization. Then, the quality of the objective is measured with respect to its performance on both model-matched and model-unmatched data.

Let us start by generating a ground truth random (weighted) tree $T \in \mathcal{T}_n$ with edge weights distributed uniformly and independently in $[0, 1]$. Measurements are subsequently generated by computing the  dissimilarities of leaves in $T$. According to the Dasgupta objective, the dissimilarity between $v_i, v_j \in \mathcal{L}(T)$ is given by $|\mathcal{L}(T(v_i \land v_j)|$, whereas according to the $\ell_2$ objective, it is given by $d_{T}(v_i, v_j)$ (all distances are computed wrt the weighted tree $T$). We postulate that \emph{if the measurements are generated according to the Dasgupta notion of dissimilarities, then his objective should perform better in terms of recovering the original tree than the $\ell_2$ objective; and vice versa}. 

In this experiment, we generate $100$ random ground truth trees with i.i.d. edge weights from $\mathrm{Unif.}[0,1]$; and produce measurements according to both objectives, denoted by $\mathcal{D}_{\mathrm{dasg.}}$ and $\mathcal{D}_{\ell_2}$. Then, we generate a {\bf fixed} set of $10^5$ random trees and find the tree that minimizes the cost function of each of the problems. Let $\mathcal{D}^{n}_{\mathrm{dasg.}}$ be the measurements of the form $\{|\mathcal{L}(T(v_i \land v_j)|\}_{i,j \in [n]}$ and $\mathcal{D}^{n}_{\ell_2}$ be the measurements of the form $\{d_{T}(v_i, v_j)\}_{i,j \in [n]}$ for the randomly generated tree $T \in \mathcal{T}_n$. We let $T_{\ell_2}(\mathcal{D}^{n}_{\mathrm{dasg.}})$ be the binary tree with the minimum $\ell_2$ cost on dissimilarity measurements $\mathcal{D}^{n}_{\mathrm{dasg.}}$. We define $T_{\ell_2}(\mathcal{D}^{n}_{\ell_2})$, $T_{\mathrm{dasg.}}(\mathcal{D}^{n}_{\mathrm{\ell_2}})$, and $T_{\mathrm{dasg.}}(\mathcal{D}^{n}_{\mathrm{dasg.}})$ similarly. The goal is to compare the deviation of $T_{\ell_2}(\mathcal{D}^{n}_{\mathrm{dasg.}})$ and $T_{\mathrm{dasg.}}(\mathcal{D}^{n}_{\mathrm{dasg.}})$ from the ground truth tree $T$, i.e., $d(T, T_{\ell_2}(\mathcal{D}^{n}_{\mathrm{dasg.}}))$ vs. $d(T, T_{\mathrm{dasg.}}(\mathcal{D}^{n}_{\mathrm{dasg.}}))$. We ask the same question for $T_{\ell_2}(\mathcal{D}^{n}_{\ell_2})$ and $T_{\mathrm{dasg.}}(\mathcal{D}^{n}_{\ell_2})$. 

In Figure~\ref{fig:exp_das_vs_l2}, we show the scatter plot of the results for each of the $100$ random trials, pertaining to trees with $n$ leaves. If the measurements are generated according to the $\ell_2$ objective, then $T_{\ell_2}$ recovers the underlying tree with higher accuracy compared to $T_{\mathrm{dasg.}}$, as desired. On the other hand, for dissimilarities generated according to Dasgupta's objective, $T_{\ell_2}$ still manages to recover the underlying tree better than $T_{\mathrm{dasg.}}$ for larger trees ($n = 10,11,12)$. This suggests that the $\ell_2$ objective could be more useful in recovering the underlying tree for more versatile types of dissimilarity measurements and larger datasets. 

\section{Hyperbolic tree-denoising}\label{sec:HyperAid}
We describe next how to solve the constrained optimization Problem~\ref{prob:l_p} in a hyperbolic space, and explain why our approach has the effect of data denoising for tree-metric learning. The HyperAid framework is depicted in Figure~\ref{fig:tree_ultrametric}. In particular, the HyperAid encoder is described in Section~\ref{sec:enc_part}, while the decoder is described in Section~\ref{sec:dec_part}. 

\subsection{A $\delta$-hyperbolic metric as a soft tree-metric}\label{sec:enc_part}

The first key idea is to exploit the fact that a tree-metric $d_T$ is $0$-hyperbolic. This can be deduced from the following theorem.
\begin{theorem}[Theorem 1 in~\cite{buneman1974note}]
A graph $G(V,E,w)$ is a tree if and only if it is connected, contains no loops and has shortest path distance $d$ satisfying the four-point condition:
\begin{align}\label{eq:four_points_true}
    & \forall x,y,z,r\in V,\quad d(x,y)+d(z,r) \nonumber\\
    & \leq \max \left\{d(x,z)+d(r,y),d(r,x)+d(y,z) \right\}.
\end{align}
\end{theorem}
Note that through some simple algebra, one can rewrite condition~\eqref{eq:four_points} for $\delta$-hyperbolicity as
\begin{align}\label{eq:four_points_2}
    & \forall x,y,z,r\in V,\quad d(x,y)+d(z,r) - 2\delta \nonumber\\
    & \leq \max \left\{d(x,z)+d(r,y),d(r,x)+d(y,z) \right\}.
\end{align}
It is also known that any $0$-hyperbolic metric space embeds isometrically into a metric-tree~\cite{dress1984trees}.

The second key idea is that condition~\eqref{eq:four_points} with a value of $\delta > 0$ may be viewed as a relaxation of the tree-metric constraint. The parameter $\delta$ can be viewed as the slackness variable, while a $\delta$-hyperbolic metric can be viewed as a ``soft'' tree-metric. Hence, the smaller the value of $\delta$, the closer the output metric to a tree-metric. Denoising therefore refers to the process of reducing $\delta$.

Note that direct testing of the condition~\eqref{eq:four_points} takes $O(n^4)$ time, which is prohibitively complex in practice. Checking the soft tree-metric constraint directly also takes $O(n^4)$ time, but to avoid this problem we exploit the connection between $\delta$-hyperbolic metric spaces and negative curvature spaces. This is our third key idea, which allows us to optimize the Problem~\ref{prob:l_p} satisfying the soft tree-metric constraint efficiently. More precisely, $\delta \times \sqrt{c} = \mathrm{const}$ (see Appendix~\ref{apx:pf_const}). Hence, a distance metric induced by a hyperbolic space of negative curvature $-c$ will satisfy the soft tree-metric constraint with slackness at most $\delta=\mathrm{const.}\times \frac{1}{\sqrt{c}}$. Therefore, using recent results pertaining to optimization in hyperbolic spaces (i.e., Riemannian SGD and Riemannian Adam~\cite{bonnabel2013stochastic,geoopt2020kochurov}), we can efficiently solve the Problem~\ref{prob:l_p} with a soft tree-metric constraint in a hyperbolic space. Combined with our previous argument, we demonstrate that finding the closest $\delta$-hyperbolic metric with $\delta$ small enough may be viewed as tree-metric denoising.

Based on $\delta \times \sqrt{c} = \mathrm{const}$, one many want to let $c\rightarrow \infty$. Unfortunately, optimization in hyperbolic spaces is known to have precision-quality tradeoff problems~\cite{pmlr-v80-sala18a,nickel2018learning}, and some recent works has attempted to address this issue~\cite{yu2019numerically,yu2021representing}. In practice, one should choose $c$ at a moderate scale ($\sim 100$ in our experiment). The loss optimized by the encoder in HyperAid equals
\begin{align}
    \min_{x_1,\ldots,x_n \in \mathbb{B}_c^d}\Big( \sum_{i,j\in [n] } |d_{\mathbb{B}_c^d}(x_i, x_j)-d_{i,j}|^p \Big)^{\frac{1}{p}},
\end{align}
where $p\geq 1$, and $d_{i,j}$ are dissimilarities between $n$ entities. 

A comparison between HyperAid and HypHC is once more in place. The purpose of using hyperbolic optimization in HypHC and our HyperAid is very different. HypHC aims to optimize the Dasgupta objective and introduces a continuous relaxation of this cost in hyperbolic space, based on continuous versions of lowest common ancestors (See Equation (9) in~\cite{chami2020trees}). In contrast, we leverage hyperbolic geometry for learning soft tree-metrics under $\ell_p$ loss, which denoises the input metrics to better fit ``tree-like'' measurements. The resulting hyperbolic metric can then be further used with any downstream tree-metric learning and inference method, and has therefore a level of universality not matched by HypHC whose goal is to perform HC. Our ablation study to follow shows that HypHC cannot offer performance improvements like our approach when combined with downstream tree-metric learners.

\subsection{Decoding the tree-metric}\label{sec:dec_part}
Tree-metric fitting methods produce trees based on arbitrary input metrics. In this context, the NJ~\cite{saitou1987neighbor} algorithm is probably the most frequently used algorithm, especially in the context of reconstructing phylogenetic trees. TreeRep~\cite{sonthalia2020tree}, a divide-and-conquer method, was also recently proposed in the machine learning literature. In contrast to linkage base methods that restrict their output to be ultrametrics, both NJ and TreeRep can output a general tree-metric. When the input metric is a tree-metric, TreeRep is guaranteed to recover it correctly in the absence of noise. It is also known that the local distortion of TreeRep is at most $\delta>0$ when the input metric is not a tree-metric but a $\delta$-hyperbolic metric. NJ can recover the ground truth tree-metric when input metric is a ``slightly'' perturbed tree-metric~\cite{atteson1997performance}. Note that methods that generate ultrametrics such as linkage-based methods and Ufit do not have similar accompanying theoretical results. 

The key idea of HyperAid is that if we are able to reduce $\delta$ for the input metric, both NJ, TreeRep (as well as other methods) should offer better performance and avoid pitfalls such as negative tree edge-weights. Since an ultrametric is a special case of a tree-metric, we also expect that our denoising process will offer empirically better performance for ultrametrics as well, which is confirmed by the results presented in the Experiments Section.

\textbf{Rooting the tree. } The NJ and TreeRep methods produce unrooted trees, which may be seen as a drawback for HC applications. This property of the resulting tree is a consequence of the fact that NJ does not assume constant rates of changes in the generative phenomena. The rooting issue may be resolved in practice through the use of outgroups, or the following simple solution known as \emph{mid-point rooting}. In mid-point rooting, the root is assumed to be the midpoint of the longest distance path between a pair of leaves; additional constraints may be used to ensure stability under limited leaf-removal. Rooting may also be performed based on priors, as is the case for outgroup rooting~\cite{hess2007empirical}.

\textbf{Complexity analysis. } We focus on the time complexity of our encoder module. Note that for the exact computation of the (Dasgupta) loss, HypHC takes $O(n^3)$ time, while our $\ell_p$ norm loss only takes $O(n^2)$ time. Nevertheless, as pointed out by the authors of HypHC, one can reduce the time complexity to $O(n^2)$ by sampling only one triplet per leaf. A similar approach can be applied in our encoder, resulting in a $O(n)$ time complexity. Hence, the time complexity of our HyperAid encoder is significantly lower than that of HypHC, as we focus on a different objective function. With respect to tree-metric and ultrametric fitting methods, we still have low-order polynomial time complexities. For example, NJ runs in $O(n^3)$ time, while an advanced implementation of UPGMA requires $O(n^2)$ time~\cite{murtagh1984complexities}.

\section{Experiments}\label{sec:exp}
To demonstrate the effectiveness of our HyperAid framework, we conduct extensive experiments on both synthetic and real-world datasets. We mainly focus on determining if denoised data leads to decoded trees that have a smaller losses  compared to those generated by decoders that operate on the input metrics directly (Direct). Due to space limitations, we only examine the $\ell_2$ norm loss, as this is the most common choice of objective in practice; HyperAid easily generalizes to arbitrary $p\in [1,\infty)$. An in-depth description of experiments is deferred to the Supplement. 

\textbf{Tree-metric methods. }We test both NJ and TreeRep methods as decoders. Note that TreeRep is a randomized algorithm so we repeat it $20$ times and pick the tree with lowest $\ell_2$ loss, as suggested in the original paper~\cite{sonthalia2020tree}. We also test the T-REX~\cite{boc2012t} version of NJ and Unweighted NJ (UNJ)~\cite{gascuel1997concerning}, both of which outperform NJ due to their sophisticated post-processing step. T-REX is not an open-source toolkit and we can only use it by interfacing with its website manually, which limits the size of the trees that can be tested; this is the current reason for reporting the performance of T-REX only for small-size real-world datasets.

\textbf{Ultrametric methods. }We also test if our framework can help ultrametric fitting methods, including linkage-based methods (single, complete, average, weight) and gradient-based ultrametric methods such as Ufit. For Ufit, we use the default hyperparameters provided by the authors~\cite{chierchia2020ultrametric}, akin to~\cite{chami2020trees}. 

\subsection{Synthetic datasets}
\begin{figure*}[th]
  \includegraphics[trim={5cm 13cm 6.5cm 2cm},clip,width=0.49\linewidth]{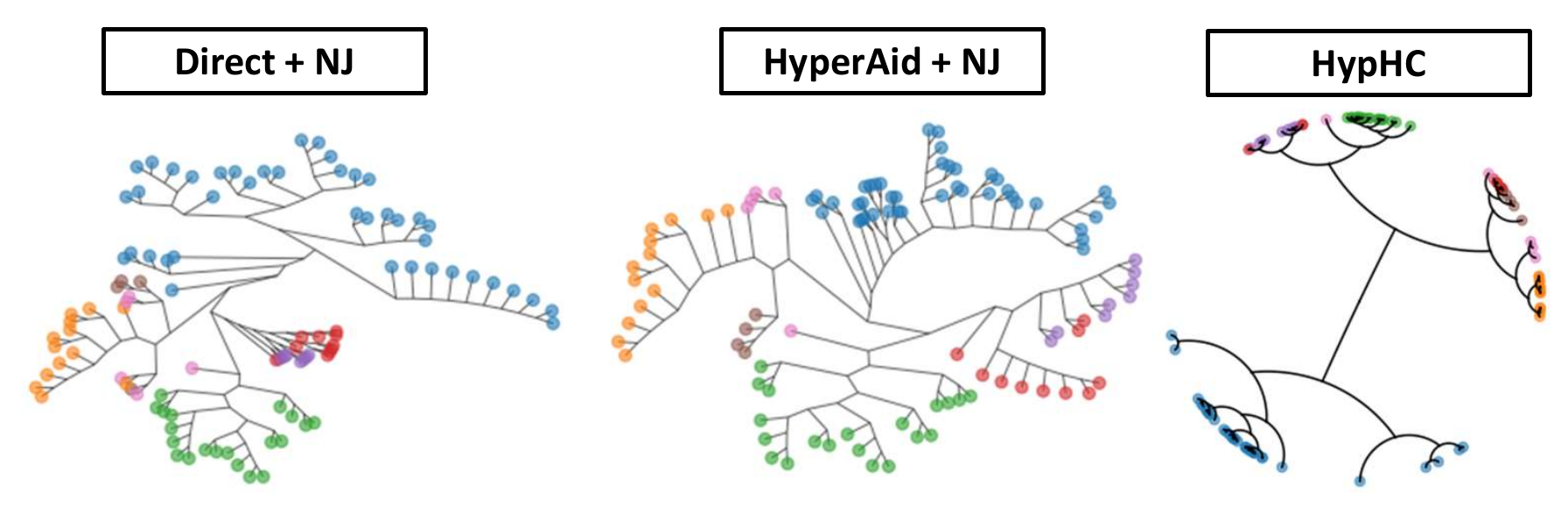}
  \includegraphics[trim={5cm 13cm 6.5cm 2cm},clip,width=0.49\linewidth]{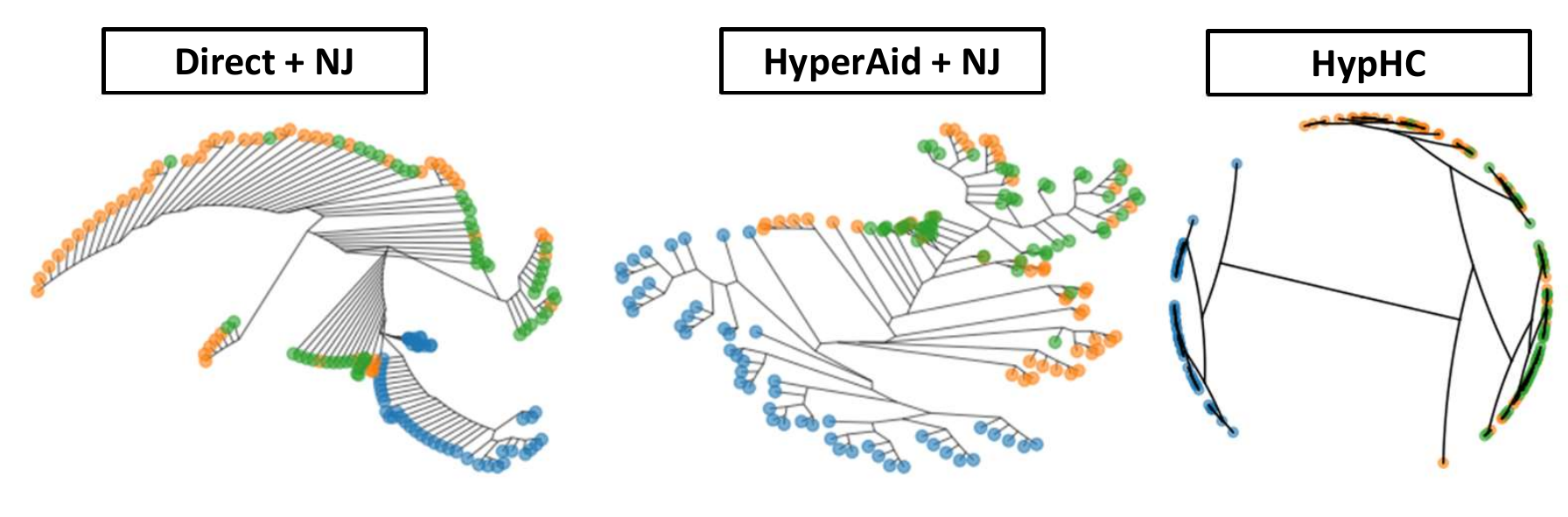}
  \vspace{-0.4cm}
  \caption{Visualization of the resulting trees for zoo (left) and iris (right). Vertex colors indicate the ground truth labels.}
  \label{fig:tree_visual}
  \vspace{-0.2cm}
\end{figure*}
\begin{table}[t]
\setlength{\tabcolsep}{1.0pt}
\caption{The $\ell_2$ norm loss of the resulting trees, averaged over three independent runs; $n$, $R$, and EL stand for the number of leaves, edge-noise ratio and the $\ell_2$ loss of the hyperbolic distances that we learned, respectively (listed for each dataset). Bold values indicate the best results among tree-metric methods and ultrametric methods. If additionally depicted in red, the values pertain to the best achievable performance. Grey shaded boxes indicate the better of the two results for Direct and HyperAid, for the same decoder. We also report the $\delta$ hyperbolicity of the raw input metric and our learned hyperbolic metric, as well as the values of Gain$=$ (Direct/HyperAid -1). Note that the first two decoders produce tree-metrics while the remaining ones result in ultrametrics.}
\vspace{-0.3cm}
\label{tab:synthetic}
\scriptsize
\begin{tabular}{@{}cccccccccc@{}}
\toprule
 &
  \multicolumn{3}{c}{$n=64,R=0.1$, EL$=32.93$} &
  \multicolumn{3}{c}{$n=64,R=0.3$, EL$=37.75$} &
  \multicolumn{3}{c}{$n=64,R=0.5$, EL$=31.36$} \\ \midrule
 &
  \begin{tabular}[c]{@{}c@{}}Direct\\ ($\delta:4.67$)\end{tabular} &
  \begin{tabular}[c]{@{}c@{}}HyperAid\\ ($\delta:3.14$)\end{tabular} &
  Gain (\%) &
  \begin{tabular}[c]{@{}c@{}}Direct\\ ($\delta:4.33$)\end{tabular} &
  \begin{tabular}[c]{@{}c@{}}HyperAid\\ ($\delta:2.60$)\end{tabular} &
  Gain (\%) &
  \begin{tabular}[c]{@{}c@{}}Direct\\ ($\delta:3.50$)\end{tabular} &
  \begin{tabular}[c]{@{}c@{}}HyperAid\\ ($\delta:2.05$)\end{tabular} &
  Gain (\%) \\ \midrule
NJ &
  139.57 &
  \cellcolor[HTML]{EFEFEF}{\color[HTML]{FE0000} \textbf{117.19}} &
  19.09 &
  97.14 &
  \cellcolor[HTML]{EFEFEF}{\color[HTML]{FE0000} \textbf{93.14}} &
  4.29 &
  80.92 &
  \cellcolor[HTML]{EFEFEF}{\color[HTML]{FE0000} \textbf{72.36}} &
  11.82 \\
TreeRep &
  167.59 &
  \cellcolor[HTML]{EFEFEF}150.37 &
  11.45 &
  169.06 &
  \cellcolor[HTML]{EFEFEF}138.28 &
  22.26 &
  131.52 &
  \cellcolor[HTML]{EFEFEF}97.15 &
  35.37 \\ \midrule
single &
  379.14 &
  \cellcolor[HTML]{EFEFEF}370.55 &
  2.31 &
  290.11 &
  \cellcolor[HTML]{EFEFEF}254.77 &
  13.87 &
  236.10 &
  \cellcolor[HTML]{EFEFEF}196.38 &
  20.22 \\
complete &
  376.24 &
  \cellcolor[HTML]{EFEFEF}355.54 &
  5.82 &
  293.06 &
  \cellcolor[HTML]{EFEFEF}279.55 &
  4.83 &
  281.92 &
  \cellcolor[HTML]{EFEFEF}231.41 &
  21.82 \\
average &
  \cellcolor[HTML]{EFEFEF}\textbf{148.68} &
  149.83 &
  -0.77 &
  108.05 &
  \cellcolor[HTML]{EFEFEF}\textbf{106.64} &
  1.32 &
  85.22 &
  \cellcolor[HTML]{EFEFEF}\textbf{83.63} &
  1.89 \\
weighted &
  162.35 &
  \cellcolor[HTML]{EFEFEF}150.86 &
  7.61 &
  132.96 &
  \cellcolor[HTML]{EFEFEF}130.85 &
  1.61 &
  111.55 &
  \cellcolor[HTML]{EFEFEF}104.41 &
  6.84 \\
Ufit &
  228.21 &
  \cellcolor[HTML]{EFEFEF}227.25 &
  0.42 &
  131.97 &
  \cellcolor[HTML]{EFEFEF}124.78 &
  5.76 &
  90.20 &
  \cellcolor[HTML]{EFEFEF}85.45 &
  5.56 \\ \midrule
 &
  \multicolumn{3}{c}{$n=128,R=0.1$, EL$=172.45$} &
  \multicolumn{3}{c}{$n=128,R=0.3$, EL$=122.12$} &
  \multicolumn{3}{c}{$n=128,R=0.5$, EL$=94.26$} \\ \midrule
 &
  \begin{tabular}[c]{@{}c@{}}Direct\\ ($\delta:6.67$)\end{tabular} &
  \begin{tabular}[c]{@{}c@{}}HyperAid\\ ($\delta:1.36$)\end{tabular} &
  Gain (\%) &
  \begin{tabular}[c]{@{}c@{}}Direct\\ ($\delta:4.83$)\end{tabular} &
  \begin{tabular}[c]{@{}c@{}}HyperAid\\ ($\delta:1.31$)\end{tabular} &
  Gain (\%) &
  \begin{tabular}[c]{@{}c@{}}Direct\\ ($\delta:4.00$)\end{tabular} &
  \begin{tabular}[c]{@{}c@{}}HyperAid\\ ($\delta:1.23$)\end{tabular} &
  Gain (\%) \\ \midrule
NJ &
  283.73 &
  \cellcolor[HTML]{EFEFEF}{\color[HTML]{FE0000} \textbf{259.38}} &
  9.38 &
  201.71 &
  \cellcolor[HTML]{EFEFEF}{\color[HTML]{FE0000} \textbf{181.52}} &
  11.12 &
  151.58 &
  \cellcolor[HTML]{EFEFEF}{\color[HTML]{FE0000} \textbf{142.86}} &
  6.09 \\
TreeRep &
  505.03 &
  \cellcolor[HTML]{EFEFEF}318.62 &
  58.50 &
  383.12 &
  \cellcolor[HTML]{EFEFEF}228.30 &
  67.81 &
  343.21 &
  \cellcolor[HTML]{EFEFEF}192.57 &
  78.22 \\ \midrule
single &
  936.98 &
  \cellcolor[HTML]{EFEFEF}767.25 &
  22.12 &
  688.52 &
  \cellcolor[HTML]{EFEFEF}519.70 &
  32.48 &
  563.58 &
  \cellcolor[HTML]{EFEFEF}431.50 &
  30.61 \\
complete &
  826.73 &
  \cellcolor[HTML]{EFEFEF}698.71 &
  18.32 &
  738.98 &
  \cellcolor[HTML]{EFEFEF}583.67 &
  26.60 &
  619.42 &
  \cellcolor[HTML]{EFEFEF}454.95 &
  36.14 \\
average &
  \cellcolor[HTML]{EFEFEF}\textbf{323.58} &
  326.26 &
  -0.82 &
  236.29 &
  \cellcolor[HTML]{EFEFEF}\textbf{234.92} &
  0.58 &
  182.29 &
  \cellcolor[HTML]{EFEFEF}\textbf{178.40} &
  2.18 \\
weighted &
  340.93 &
  \cellcolor[HTML]{EFEFEF}340.30 &
  0.18 &
  \cellcolor[HTML]{EFEFEF}301.30 &
  316.85 &
  -4.90 &
  \cellcolor[HTML]{EFEFEF}262.61 &
  273.47 &
  -3.97 \\
Ufit &
  681.52 &
  \cellcolor[HTML]{EFEFEF}556.37 &
  22.49 &
  388.47 &
  \cellcolor[HTML]{EFEFEF}303.14 &
  28.15 &
  231.69 &
  \cellcolor[HTML]{EFEFEF}198.29 &
  16.84 \\ \midrule
\multicolumn{1}{l}{} &
  \multicolumn{3}{c}{$n=256,R=0.1$, EL$=248.89$} &
  \multicolumn{3}{c}{$n=256,R=0.3$, EL$=219.41$} &
  \multicolumn{3}{c}{$n=256,R=0.5$, EL$=172.52$} \\ \midrule
 &
  \begin{tabular}[c]{@{}c@{}}Direct\\ ($\delta:7.33$)\end{tabular} &
  \begin{tabular}[c]{@{}c@{}}HyperAid\\ ($\delta:4.24$)\end{tabular} &
  Gain (\%) &
  \begin{tabular}[c]{@{}c@{}}Direct\\ ($\delta:5.83$)\end{tabular} &
  \begin{tabular}[c]{@{}c@{}}HyperAid\\ ($\delta:3.10$)\end{tabular} &
  Gain (\%) &
  \begin{tabular}[c]{@{}c@{}}Direct\\ ($\delta:4.33$)\end{tabular} &
  \begin{tabular}[c]{@{}c@{}}HyperAid\\ ($\delta:2.38$)\end{tabular} &
  Gain (\%) \\ \midrule
NJ &
  910.31 &
  \cellcolor[HTML]{EFEFEF}{\color[HTML]{FE0000} \textbf{642.68}} &
  41.64 &
  450.57 &
  \cellcolor[HTML]{EFEFEF}{\color[HTML]{FE0000} \textbf{420.97}} &
  7.03 &
  320.29 &
  \cellcolor[HTML]{EFEFEF}{\color[HTML]{FE0000} \textbf{312.83}} &
  2.38 \\
TreeRep &
  1484.36 &
  \cellcolor[HTML]{EFEFEF}1078.49 &
  37.63 &
  1125.81 &
  \cellcolor[HTML]{EFEFEF}751.72 &
  49.76 &
  925.65 &
  \cellcolor[HTML]{EFEFEF}551.10 &
  67.96 \\ \midrule
single &
  2343.78 &
  \cellcolor[HTML]{EFEFEF}2277.91 &
  2.89 &
  1666.06 &
  \cellcolor[HTML]{EFEFEF}1397.08 &
  19.25 &
  1313.64 &
  \cellcolor[HTML]{EFEFEF}1066.32 &
  23.19 \\
complete &
  2186.13 &
  \cellcolor[HTML]{EFEFEF}2024.22 &
  7.99 &
  1786.13 &
  \cellcolor[HTML]{EFEFEF}1532.99 &
  16.51 &
  1565.70 &
  \cellcolor[HTML]{EFEFEF}1390.28 &
  12.61 \\
average &
  758.35 &
  \cellcolor[HTML]{EFEFEF}\textbf{756.91} &
  0.18 &
  \cellcolor[HTML]{EFEFEF}\textbf{521.83} &
  521.91 &
  -0.01 &
  392.09 &
  \cellcolor[HTML]{EFEFEF}\textbf{386.87} &
  1.35 \\
weighted &
  794.90 &
  \cellcolor[HTML]{EFEFEF}776.84 &
  2.32 &
  \cellcolor[HTML]{EFEFEF}681.73 &
  683.01 &
  -0.18 &
  626.06 &
  \cellcolor[HTML]{EFEFEF}600.36 &
  4.27 \\
Ufit &
  2026.42 &
  \cellcolor[HTML]{EFEFEF}1974.59 &
  2.62 &
  1206.72 &
  \cellcolor[HTML]{EFEFEF}1008.93 &
  19.60 &
  774.92 &
  \cellcolor[HTML]{EFEFEF}641.07 &
  20.88 \\ \bottomrule
\end{tabular}
\vspace{-0.5cm}
\end{table}

Let $R\in \{0.1,0.3,0.5\}$ be what we refer to the \emph{edge-noise rate}, which controls the number of random edges added to a tree as a means of introducing perturbations. We generate a random binary tree and then add $R\times (2n-2)$ random edges between vertices, including both leaves and internal vertices. The input metric to our encoder is the collection of shortest path distances between leaves. We generate such corrupted random binary trees with $n\in \{64,128,256\}$. All the results are averages over three independent runs with the same set of input metrics. Note that this edge-noise model results in graph-generated similarities which may not have full generality but are amenable to simple evaluations of HyperAid in terms of detecting spurious noisy edges and removing their effect from the resulting tree-fits or hierarchies.

It is easy to see that our framework consistently improves the performance of tree-metric learning methods, with gains as large as $41\%$ for NJ and $78\%$ for TreeRep (and $12.54\%$ and $47.66\%$ on average, respectively). Furthermore, one can see that the hyperbolic metric learned by HyperAid indeed results in smaller $\delta$-hyperbolicity values compared to the input. For ultrametric methods, our framework frequently leads to improvements but not of the same order as those seen for general tree-metrics. We also test the centroid, median and ward linkage methods for raw input metric data, but given that these methods requires the input metric to be Euclidean we could not combine them with HyperAid. Nevertheless, their results are significantly worse than those offered by our framework and thus not included in Table~\ref{tab:synthetic}.

\vspace{-0.15cm}
\subsection{Real-world datasets}\label{sec:exp_real}
We examine the effectiveness of our HyperAid framework on five standard datasets from the UIC Machine Learning repository~\cite{Dua:2019}. We use $1-$cosine similarities as the raw input metrics. The results are listed in Table~\ref{tab:realworld}.

One can see that the HyperAid framework again consistently and significantly improves the performance of all downstream methods (reaching gains in excess of $120\%$). Surprisingly, we find that average linkage, the best ultrametric fitting method, achieves smaller $\ell_2$ norm losses on the Spambase dataset when compared to tree-metric methods like NJ. We conjecture two possible explanations: 1) We have not explored more powerful tree-metric reconstruction methods like T-REX. Again, the reason for not being able to make full use of T-REX is the fact that the software requires interactive data loading and processing; 2) The optimal tree-metric is very close to an ultrametric. 
It is still worth pointing out that despite the fact that the ultrametric fits have smaller loss than tree-metric fits, even in the former case HyperAid consistently improved all tested algorithms on real-world datasets.

\vspace{-0.15cm}
\subsection{Ablation study}

While minimizing the $\ell_2$ loss in the encoding process of the HyperAid framework is an intuitively good choice given that our final tree-metric fit is evaluated based on the $\ell_2$ loss, it is still possible to choose other objective functions (such as the Dasguptas loss) for the encoder module. We therefore compare the performance of the hyperbolic metric learned using HypHC, centered on the Dasgupta objective, with the performance of our HyperAid encoder coupled with several possible decoding methods. For HypHC, we used the official implementation and hyperparameters provided by the authors~\cite{chami2020trees}. The results are listed in Table~\ref{tab:ablation} in the Supplement. The results reveal that minimizing the $\ell_2$ loss during hyperbolic embedding indeed offers better results than those obtained by using the Dasgupta loss when our ultimate performance criterion is the $\ell_2$ loss. Note that the default decoding rule of HypHC produces unweighted trees and hence offers worse performance than the tested mode (and the results are therefore omitted from the table).

\vspace{-0.15cm}
\subsection{Tree visualization and analysis}

One conceptually straightforward method for evaluating the quality of the hierarchical clustering is to use the ground truth labels of object classes available for real-world datasets such as Zoo and Iris.

The Zoo dataset comprises $7$ classes of the animal kingdom, including insects, cartilaginous and bony fish, amphibians, reptiles, birds and mammals. The hierarchical 
clusterings based on NJ, HyperAid+NJ and HypHC reveal that the first and last method tend to misplace amphibians and bony fish as well as insects and cartilaginous fish by placing them in 
the same subtrees. The NJ method also shows a propensity to create ``path-like trees'' which do not adequately capture known phylogenetic information (e.g., for mammals in particular). 
On the other hand, HyperAid combined with NJ discriminates accurately between all seven classes with the exception of two branches of amphibian species.  

The Iris dataset includes three groups of iris flowers, Iris-setosa, Iris-versicolor and Iris-virginica. All three methods perfectly identify the subtrees of Iris-setosa species, but once again the NJ method produces a ``path-like tree'' for all three groups which does not conform to nonlinear models of evolution. It is well-known that it is hard to distinguish Iris versicolor and Iris virginica flowers as both have similar flower color and growth/bloom times. As a result, all three methods result in subtrees containing intertwined clusters of flowers from both groups.

\begin{table*}[t]
\setlength{\tabcolsep}{1.0pt}
\caption{The $\ell_2$ norm loss of the resulting trees (see the caption of Table~\ref{tab:synthetic} for the format description). In addition, we use $\delta^\star$ to indicate that the $\delta$ value is approximated by random sampling of one million quadruples of leaves. NA refers to ``not applicable''. Note that the first four decoders produce tree-metrics while the remaining produce ultrametrics.}
\vspace{-0.2cm}
\label{tab:realworld}
\scriptsize
\begin{tabular}{@{}ccccccccccccccccc@{}}
\toprule
 &
  \multicolumn{3}{c}{Zoo, $n=101$, EL$=7.85$} &
  \multicolumn{3}{c}{Iris, $n=150$, EL$=15.64$} &
  \multicolumn{3}{c}{Glass, $n=214$, EL$=27.58$} &
  \multicolumn{3}{c}{Segmentation, $n=2310$, EL$=248.40$} &
  \multicolumn{3}{c}{Spambase, $n=4601$, EL$=255.22$} &
  Average Gain (\%)\\ \midrule
 &
  \begin{tabular}[c]{@{}c@{}}Direct\\ ($\delta=0.349$)\end{tabular} &
  \begin{tabular}[c]{@{}c@{}}HyperAid\\ ($\delta=0.048$)\end{tabular} &
  Gain (\%) &
  \begin{tabular}[c]{@{}c@{}}Direct\\ ($\delta=0.489$)\end{tabular} &
  \begin{tabular}[c]{@{}c@{}}HyperAid\\ ($\delta=0.062$)\end{tabular} &
  Gain (\%) &
  \begin{tabular}[c]{@{}c@{}}Direct\\ ($\delta=0.453$)\end{tabular} &
  \begin{tabular}[c]{@{}c@{}}HyperAid\\ ($\delta=0.044$)\end{tabular} &
  Gain (\%) &
  \begin{tabular}[c]{@{}c@{}}Direct\\ ($\delta^\star=0.389$)\end{tabular} &
  \begin{tabular}[c]{@{}c@{}}HyperAid\\ ($\delta^\star=0.133$)\end{tabular} &
  Gain (\%) &
  \begin{tabular}[c]{@{}c@{}}Direct\\ ($\delta^\star=0.234$)\end{tabular} &
  \begin{tabular}[c]{@{}c@{}}HyperAid\\ ($\delta^\star=0.065$)\end{tabular} &
  Gain (\%) \\ \midrule
NJ &
  11.77 &
  \cellcolor[HTML]{EFEFEF}8.99 &
  30.94 &
  43.85 &
  \cellcolor[HTML]{EFEFEF}18.17 &
  141.22 &
  60.96 &
  \cellcolor[HTML]{EFEFEF}30.76 &
  98.15 &
  856.47 &
  \cellcolor[HTML]{EFEFEF}{\color[HTML]{FE0000} \textbf{292.78}} &
  192.53 &
  1008.6 &
  \cellcolor[HTML]{EFEFEF}{\textbf{377.93}} &
  166.87 &
  125.94\\
TreeRep &
  14.65 &
  \cellcolor[HTML]{EFEFEF}9.32 &
  57.12 &
  22.59 &
  \cellcolor[HTML]{EFEFEF}19.00 &
  18.89 &
  71.81 &
  \cellcolor[HTML]{EFEFEF}32.00 &
  124.41 &
  1357.28 &
  \cellcolor[HTML]{EFEFEF}1038.98 &
  30.64 &
  2279.31 &
  \cellcolor[HTML]{EFEFEF}927.02 &
  145.87 &
  75.38\\
NJ(T-REX) &
  8.85 &
  \cellcolor[HTML]{EFEFEF}8.59 &
  3.09 &
  26.38 &
  \cellcolor[HTML]{EFEFEF}16.96 &
  55.51 &
  37.61 &
  \cellcolor[HTML]{EFEFEF}29.58 &
  27.11 &
  NA &
  NA &
  NA &
  NA &
  NA &
  NA &
  28.57\\
UNJ(T-REX) &
  8.60 &
  \cellcolor[HTML]{EFEFEF}{\color[HTML]{FE0000} \textbf{8.57}} &
  0.31 &
  17.47 &
  \cellcolor[HTML]{EFEFEF}{\color[HTML]{FE0000} \textbf{16.86}} &
  3.58 &
  29.41 &
  \cellcolor[HTML]{EFEFEF}{\color[HTML]{FE0000} \textbf{29.38}} &
  0.09 &
  NA &
  NA &
  NA &
  NA &
  NA &
  NA &
  1.33\\ \midrule
single &
  34.53 &
  \cellcolor[HTML]{EFEFEF}17.17 &
  101.11 &
  84.66 &
  \cellcolor[HTML]{EFEFEF}65.97 &
  28.33 &
  95.91 &
  \cellcolor[HTML]{EFEFEF}52.82 &
  81.55 &
  1174.95 &
  \cellcolor[HTML]{EFEFEF}924.76 &
  27.06 &
  1809.83 &
  \cellcolor[HTML]{EFEFEF}1109.75 &
  63.08 &
  60.23\\
complete &
  20.78 &
  \cellcolor[HTML]{EFEFEF}10.64 &
  95.27 &
  58.13 &
  \cellcolor[HTML]{EFEFEF}30.95 &
  87.80 &
  88.19 &
  \cellcolor[HTML]{EFEFEF}37.75 &
  133.61 &
  928.63 &
  \cellcolor[HTML]{EFEFEF}486.31 &
  90.96 &
  1248.15 &
  \cellcolor[HTML]{EFEFEF}633.70 &
  96.96 &
  100.92\\
average &
  10.04 &
  \cellcolor[HTML]{EFEFEF}{ \textbf{9.17}} &
  9.47 &
  23.38 &
  \cellcolor[HTML]{EFEFEF}{ \textbf{22.99}} &
  1.68 &
  31.94 &
  \cellcolor[HTML]{EFEFEF}{ \textbf{30.79}} &
  3.73 &
  339.41 &
  \cellcolor[HTML]{EFEFEF}{ \textbf{300.07}} &
  13.11 &
  365.78 &
  \cellcolor[HTML]{EFEFEF}{\color[HTML]{FE0000} \textbf{350.68}} &
  4.31 &
  6.46\\
weighted &
  12.08 &
  \cellcolor[HTML]{EFEFEF}11.05 &
  9.33 &
  \cellcolor[HTML]{EFEFEF}27.20 &
  29.07 &
  -6.42 &
  34.75 &
  \cellcolor[HTML]{EFEFEF}34.16 &
  1.74 &
  458.17 &
  \cellcolor[HTML]{EFEFEF}337.75 &
  35.66 &
  385.37 &
  \cellcolor[HTML]{EFEFEF}382.13 &
  0.84 &
  8.22\\
Ufit &
  10.53 &
  \cellcolor[HTML]{EFEFEF}9.35 &
  12.70 &
  26.02 &
  \cellcolor[HTML]{EFEFEF}23.02 &
  13.03 &
  33.21 &
  \cellcolor[HTML]{EFEFEF}31.04 &
  6.98 &
  354.66 &
  \cellcolor[HTML]{EFEFEF}314.08 &
  12.92 &
  381.68 &
  \cellcolor[HTML]{EFEFEF}377.67 &
  1.06 &
  9.34\\ \bottomrule
\end{tabular}
\vspace{-0.2cm}
\end{table*}
\vspace{-0.1cm}
\begin{acks}
The work was supported by the NSF CIF program under grant number 19566384.
\end{acks}

\bibliographystyle{ACM-Reference-Format}
\bibliography{sample-base}

\appendix

\section{Proof of Proposition~\ref{prop:convex}}
\begin{proposition}
The objective in Problem~\ref{prob:l_p} is a convex function of the edge weights of the tree $T$.	
\end{proposition}
\begin{proof}
	The objective in Problem~\ref{prob:l_p} can be written as\
	\[
		\min_{A_T \in \mathcal{A} , w \in \mathbb{R}^{2n-2}_{+} } \| A_T w - d\|_p
	\]
	where $\mathcal{A}$ is the set of admissible linear operators that map edge weights to additive pairwise distances on a binary tree with $N$ leaves. Let $f(w) = \| A_T w - d\|_p$. For a $\lambda \in [0, 1]$, we have
	\begin{align*}
		f(\lambda w + (1-\lambda )v ) &= \| \lambda A_T w + (1- \lambda) A_T v - d\|_p \\
		 &= \| \lambda (A_T w-d) + (1- \lambda) (A_T v - d)\|_p \\
		 &\leq \lambda \| A_T w-d\|_{p} + (1- \lambda)\|A_T v - d\|_p \\
		 &\leq  \lambda f(w) + (1- \lambda)f(v).
	\end{align*}
	This result follows from the convexity of the $\ell_p$ norm for $p \geq 1$.
\end{proof}

\section{Proof of $\delta \times \sqrt{c} = \mathrm{const}$}\label{apx:pf_const}
\begin{theorem}
  $\delta \times \sqrt{c} = \mathrm{const}$.
\end{theorem}
\begin{proof}
Let us define a one-to-one map $\phi: \mathbb{B}_{1}^d  \rightarrow \mathbb{B}_{c}^d$, where $\phi(x) = \frac{1}{\sqrt{c}}x$. The Gromov hyperbolicity of the Poincar\'e ball is finite~\cite{vaisala2005gromov} and it scales linearly with the pairwise distances; see Definition \ref{def:gromov_hyperbolicity}. Since $d\big( \phi(x),\phi(y) \big) = \frac{1}{\sqrt{c}}d(x,y)$ for all $x,y \in \mathbb{B}_{1}^{d}$, the Gromov hyperbolicity of $\mathbb{B}_{c}^{d}$ equals $\mathrm{const.} \times \frac{1}{\sqrt{c}}$.
\end{proof}

\section{Additional experiment details}
\subsection{Detailed description of the encoder}
We implemented our framework in Python. The encoder part is developed in Pytorch, and is a modification based on the official implementation of HypHC\footnote{https://github.com/HazyResearch/HypHC}.

\textbf{Training details. } Note that all hyperparameters of each dataset are described in our implementation. Here, we briefly describe the most important ones. All the hyperparameters are tuned based on the loss of Hyp+NJ. Note that since we are solving a search problem, this is a standard approach -- for combinatorial search algorithms, one usually reports the best loss for different random seeds. HypHC also report the best loss results over several random seeds~\cite{chami2020trees}. 

When training hyperbolic embeddings in the Poincar\'e model, we follow the suggestion stated in~\cite{nickel2017poincare}. We initialize the embeddings within a small radius \verb|1e-6| with small learning rate for a few (\verb|burnin|) epochs, then increase the learning rate by a \verb|burnin_factor|. We also tune the negative curvature \verb|c| for different datasets. As mentioned in the main text, although choosing larger \verb|c| makes the resulting hyperbolic metric more similar to the tree-metric (i.e. $\delta$ becomes smaller based on the relation $\delta \times \sqrt{c} = \mathrm{const}$), it can make the optimization process harder and can potentially cause numerical issues. For synthetic datasets, we find that scaling the input metric by a factor \verb|scaling_factor| during training (and scale it back before decoding) can help with improving the final results. We conjecture that this is due to the numerical precision issues associated with hyperbolic embeddings approaches for which the resulting points are close to the boundary of Poincar\'e ball. We use Riemannian Adam implemented by the authors of HypHC as our optimizer. 

\subsection{Detailed description of the decoder}
For NJ and TreeRep, we use the implementation from the TreeRep paper\footnote{\url{https://github.com/rsonthal/TreeRep}}. We manually interact with the T-REX website\footnote{www.trex.uqam.ca} for NJ-TREX and UNJ-TREX. For linkage-based methods, we use the scipy library\footnote{\url{https://docs.scipy.org/doc/scipy/reference/generated/scipy.cluster.hierarchy.linkage.html}}. For UFit, we use the official implementation\footnote{https://github.com/PerretB/ultrametric-fitting} along with default hyperparameters.

\textbf{Remarks. }The NJ implementation, which is based on the PhyloNetworks library\footnote{http://crsl4.github.io/PhyloNetworks.jl/latest} in Julia, is a naive version of the method. It is known from the literature that such versions of NJ can potentially produce negative edge weights\footnote{\url{https://en.wikipedia.org/wiki/Neighbor_joining/\#Advantages_and_disadvantages}}. In the naive version of NJ, one simply replaces the negative edge-weights by $0$. On the other hand, a more sophisticated cleaning approach is implemented in T-REX, which is the reason why the T-REX version offers better performance. Note that TreeRep is also known to have the negative edge-weights problem, where the naive correction (replacing them with $0$) was adopted by the authors~\cite{sonthalia2020tree}.

\subsection{Additional results}
\begin{figure}[t]
  \includegraphics[trim={0cm 13cm 12cm 2cm},clip,width=\linewidth]{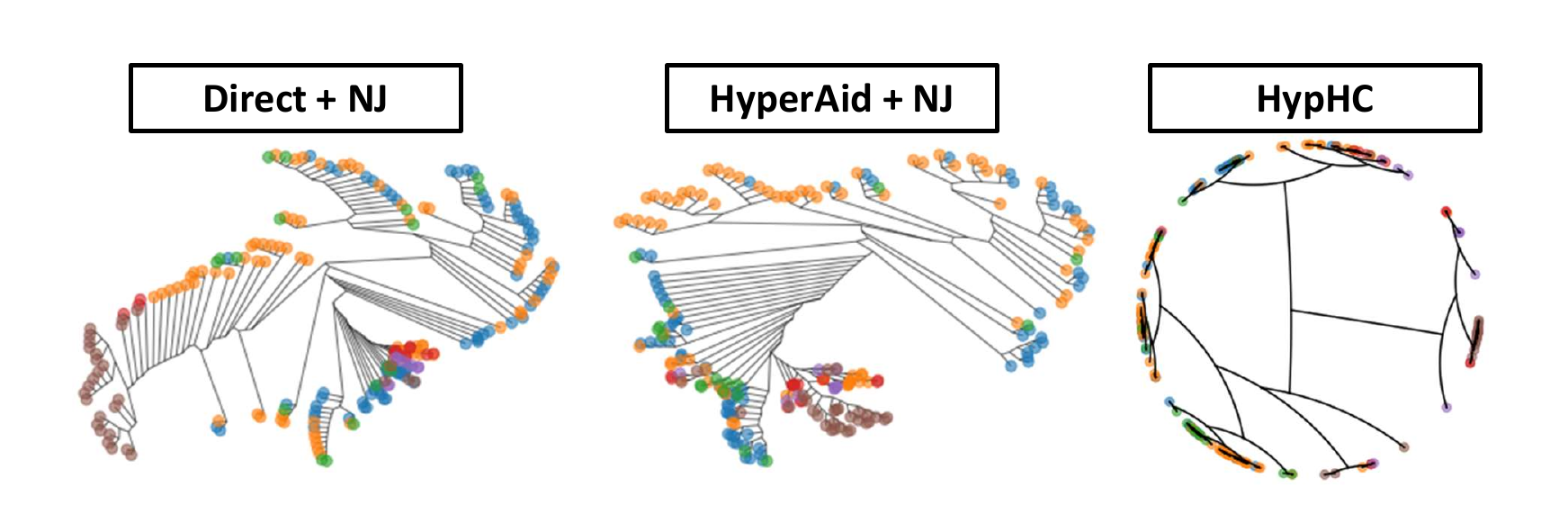}
  \caption{Visualization of resulting trees for Glass. Vertex colors indicate the ground truth labels.}
  \label{fig:tree_visual_glass}
\end{figure}

\begin{table}[t]
\caption{The $\ell_2$ losses of resulting trees generated as part of the ablation study pertaining to methods for generating hyperbolic embeddings.}
\label{tab:ablation}
\scriptsize
\begin{tabular}{@{}ccccccc@{}}
\toprule
         & \multicolumn{2}{c}{Zoo}               & \multicolumn{2}{c}{Iris}              & \multicolumn{2}{c}{Glass}              \\ \midrule
         & HyperAid                          & HypHC & HyperAid                          & HypHC & HyperAid                          & HypHC  \\ \midrule
NJ       & \cellcolor[HTML]{EFEFEF}11.77 & 31.55 & \cellcolor[HTML]{EFEFEF}18.18 & 70.43 & \cellcolor[HTML]{EFEFEF}30.76 & 87.75  \\
TreeRep  & \cellcolor[HTML]{EFEFEF}14.65 & 28.51 & \cellcolor[HTML]{EFEFEF}19.00 & 64.61 & \cellcolor[HTML]{EFEFEF}32.00 & 80.31  \\
single   & \cellcolor[HTML]{EFEFEF}34.53 & 45.01 & \cellcolor[HTML]{EFEFEF}65.97 & 85.40 & \cellcolor[HTML]{EFEFEF}52.83 & 108.44 \\
complete & \cellcolor[HTML]{EFEFEF}20.78 & 24.01 & \cellcolor[HTML]{EFEFEF}30.96 & 59.60 & \cellcolor[HTML]{EFEFEF}37.75 & 69.54  \\
average  & \cellcolor[HTML]{EFEFEF}10.05 & 33.31 & \cellcolor[HTML]{EFEFEF}23.00 & 72.36 & \cellcolor[HTML]{EFEFEF}30.79 & 90.58  \\
weighted & \cellcolor[HTML]{EFEFEF}12.09 & 32.97 & \cellcolor[HTML]{EFEFEF}29.07 & 71.00 & \cellcolor[HTML]{EFEFEF}34.16 & 87.15  \\
Ufit     & \cellcolor[HTML]{EFEFEF}10.54 & 33.07 & \cellcolor[HTML]{EFEFEF}23.03 & 73.19 & \cellcolor[HTML]{EFEFEF}31.04 & 90.34  \\ \bottomrule
\end{tabular}
\end{table}

\textbf{Tree visualization for the Glass dataset. }Clustering results obtained for the Glass dataset comprising seven classes of material are hardest to interpret, as all three methods only accurately cluster tableware glass. HypAid appears to better delineate between building-windows glass that is float-processed from building-windows glass that is nonfloat-processed.

\end{document}